\pdfoutput=1
\documentclass{article}
\usepackage{spconf,graphicx}
\usepackage{amsmath}
\usepackage{amsthm}
\usepackage{amsfonts}
\newtheorem{theorem}{Theorem}[section]

\newtheorem{proposition}[theorem]{Proposition}
\newtheorem{fact}[theorem]{Fact}
\newtheorem{definition}[theorem]{Definition}

\usepackage{color}
\usepackage{xcolor}
\usepackage{subfigure}
\usepackage{booktabs}
\usepackage{multirow}
\usepackage{float}

\title{Symplectic Structure-Aware Hamiltonian (Graph) Embeddings}
\name{Jiaxu Liu$^1$, Xinping Yi$^2$, Tianle Zhang$^1$ and Xiaowei Huang$^1$}
\address{$^1$University of Liverpool, UK and $^2$Southeast University, China}
\begin{document}
%
\maketitle
\begin{abstract}
In traditional Graph Neural Networks (GNNs), the assumption of a fixed embedding manifold often limits their adaptability to diverse graph geometries. Recently, Hamiltonian system-inspired GNNs have been proposed to address the dynamic nature of such embeddings by incorporating physical laws into node feature updates. We present Symplectic Structure-Aware Hamiltonian GNN (SAH-GNN), a novel approach that generalizes Hamiltonian dynamics for more flexible node feature updates. Unlike existing Hamiltonian approaches, SAH-GNN employs Riemannian optimization on the symplectic Stiefel manifold to adaptively learn the underlying symplectic structure, circumventing the limitations of existing Hamiltonian GNNs that rely on a pre-defined form of standard symplectic structure. This innovation allows SAH-GNN to automatically adapt to various graph datasets without extensive hyperparameter tuning. Moreover, it conserves energy during training meaning the implicit Hamiltonian system is physically meaningful. Finally, we empirically validate SAH-GNN's superiority and adaptability in node classification tasks across multiple types of graph datasets.
\end{abstract}
\section{Introduction}
\label{sec:intro}
The advent of Graph Neural Networks (GNNs) has opened up vast possibilities in various applications, ranging from social network analysis \cite{kipf2016semi} to bioinformatics \cite{zitnik2019evolution} and computer vision \cite{qin2022fusing}. Vanilla GNNs assume a fixed embedding manifold, which limits their ability to capture the complex and changing nature of graph data with different geometric priors.

Recent advances have sought to alleviate these constraints by drawing inspiration from Hamiltonian systems \cite{greydanus2019hamiltonian, kang2023node}, a class of dynamical systems governed by Hamilton's equations. The Hamiltonian perspective allows for the modeling of dynamic systems where energy is conserved, offering a more flexible approach to node feature updates. These Hamiltonian-inspired GNNs, while innovative, still face limitations due to their reliance on a pre-defined standard symplectic structure, which describes the geometry of the phase space in which the system evolves.

In this paper, we propose a \textbf{S}ymplectic Structure-\textbf{A}ware \textbf{H}amiltonian GNN (SAH-GNN) that generalizes the concept of a Hamiltonian system to better suit the geometry of graph embeddings. Unlike existing approaches that are restricted to a fixed symplectic structure, SAH-GNN utilizes Riemannian optimization techniques to learn the underlying symplectic structure of the Hamiltonian motion of node embeddings during training. This allows our model to adapt to various types of graph data without extensive hyperparameter tuning. Furthermore, SAH-GNN ensures the physical integrity of the modeled Hamiltonian system by conserving energy throughout training. Our experimental results validate the efficacy of SAH-GNN in node classification tasks across various types of graph datasets, demonstrating superior performance and adaptability compared to state-of-the-art GNN models. 

In summary, this work makes the following contributions: \textit{(1)} We propose SAH-GNN, a Hamiltonian GNN that models the evolution of node features as Hamiltonian orbits over time where the motion is governed by optimizeable symplectic transformations. It provides improved flexibility in capturing different graph geometries; \textit{(2)} We introduce Riemannian optimization on the symplectic Stiefel manifold to approximate the optimal symplectic structure adaptively during training, which is a first attempt in the domain of GNNs; \textit{(3)} Empirically, we validate the energy conservation of SAH-GNN, ensuring that the modeled Hamiltonian system is physically meaningful. Additionally, we demonstrate the superior performance and adaptability of SAH-GNN on node classification tasks using different types of graph datasets.

\section{Preliminaries}
\subsection{Symplectic Group \& Symplectic Stiefel Manifold}
Let $\mathrm{Sp}(2n)$ be the group of real matrices satisfying the symplectic constraint, expressed as
\begin{align}
    \mathrm{Sp}(2n) := \{ X\in\mathbb{R}^{2n\times 2n}\mid X^\top J_{2n}X = J_{2n} \},
\end{align}
where $J_{2n}$ is a nonsingular skew-symmetric matrix, often defined to be $J_{2n}=\begin{bmatrix}
    0 & I_n\\
    -I_n& 0
\end{bmatrix}$. A matrix in $\mathrm{Sp}(2n)$ is referred to as a \textit{symplectic matrix}. More generally, we employ the set of rectangular symplectic matrices in $\mathbb{R}^{2n\times 2k}$, which is also referred to as \textit{symplectic stiefel manifold}, defined by
\begin{align}
    \mathrm{Sp}(2n, 2k) := \{ X\in\mathbb{R}^{2n\times 2k}\mid X^\top J_{2n}X = J_{2k} \},
\end{align}
for some $k$ with $1\le k\le n$. As shown in \cite{gao2021riemannian}, $\mathrm{Sp}(2n, 2k)$ is an embedded Riemannian manifold of $\mathbb{R}^{2n\times 2k}$. The tangent space on $X\in \mathrm{Sp}(2n, 2k)$ is defined by
\begin{align}
    \mathcal{T}_X\mathrm{Sp}(2n, 2k) := \{Z \in\mathbb{R}^{2n\times 2k} \mid Z^\top JX + X^\top JZ = 0\}. \label{eq:tangent-set-sp}
\end{align}

\subsection{Hamiltonian Graph Neural Network}
\label{sec:hamgnn}
The Hamiltonian orbit $(\mathbf{q}(t), \mathbf{p}(t))$ is a set of coordinates where $\mathbf{q}\in \mathbb{R}^{k}$ and $\mathbf{p}\in \mathbb{R}^{k}$ denote the positions of a set of objects and their momentum, respectively. Define a Hamiltonian scalar function $\mathcal{H}(\mathbf{q}(t), \mathbf{p}(t))$ \textit{s.t.}
\begin{align}
    \mathbf{S}_{\mathcal{H}}(t) = \begin{bmatrix}
        \frac{\partial \mathbf{q}(t)}{\partial t}\\
        \frac{\partial \mathbf{p}(t)}{\partial t}
    \end{bmatrix} = 
    \begin{bmatrix}
    0 & I_n\\
    -I_n& 0
\end{bmatrix} \begin{bmatrix}
        \frac{\partial \mathcal{H}}{\partial \mathbf{q}}\\
        \frac{\partial \mathcal{H}}{\partial \mathbf{p}}
    \end{bmatrix},
    \label{eq:symplectic-gradient}
\end{align}
where the multiplier is the standard symplectic matrix $J_{2n}$. Formally, $\mathbf{S}_{\mathcal{H}}$ is a time-dependent vector field called \textit{symplectic gradient}. Forming a gradient flow, $\mathbf{S}_{\mathcal{H}}$ can be numerically integrated via Neural ODE \cite{chen2018neural} to get the evolved coordinates
\begin{align}
    \begin{bmatrix}
        \mathbf{q}(t_1)\\
        \mathbf{p}(t_1)
    \end{bmatrix} = 
    \begin{bmatrix}
        \mathbf{q}(t_0)\\
        \mathbf{p}(t_0)
    \end{bmatrix} + \int_{t_0}^{t_1} \mathbf{S}_{\mathcal{H}}(t) dt. \label{eq:hnn-integration}
\end{align}
A Hamiltonian Graph Neural Network (HamGNN) \cite{kang2023node} considers $\mathbf{q}(t_0)\in \mathbb{R}^{n\times k}$ as the input of each layer, $\mathbf{p}(t_0)=\mathcal{Q}_\phi(\mathbf{q}(t_0))\in \mathbb{R}^{n\times k}$ as the associated momentum vector where $\mathcal{Q}_\phi:\mathbb{R}^{k}\to\mathbb{R}^{k}$ is specified by a neural network. With $\mathcal{H}_\theta: (\mathbf{q}, \mathbf{p})\to \mathbb{R}$ also being a neural network parameterized \textbf{energy} function, the layer-wise embedding can be obtained via integrating Eq.~(\ref{eq:hnn-integration}). Finally, the integrated $\mathbf{q}(t_1)$ is taken as the embedding for input to next layer or downstream tasks. 
\begin{theorem}[Conservation law \cite{da2001lectures}]
    $\mathcal{H}(\mathbf{q}(t), \mathbf{p}(t))$ is constant along the Hamiltonian orbit if the transformation is given by a standard symplectic as in Eq.~(\ref{eq:symplectic-gradient}).
    \label{thm:conservation}
\end{theorem}
\begin{proof}
    From Eq.~(\ref{eq:symplectic-gradient}), we have:
\begin{align}
    \frac{\partial \mathbf{q}(t)}{\partial t} = \frac{\partial \mathcal{H}}{\partial \mathbf{p}(t)}, && \frac{\partial \mathbf{p}(t)}{\partial t} =  - \frac{\partial \mathcal{H}}{\partial \mathbf{q}(t)}.
\end{align}
Thus the total derivative of the Hamiltonian is given by
\begin{align}
    \frac{\partial \mathcal{H}}{\partial t} &= \frac{\partial \mathcal{H}}{\partial \mathbf{q}(t)} \cdot \frac{\partial \mathbf{q}(t)}{\partial t} + \frac{\partial \mathcal{H}}{\partial \mathbf{p}(t)} \cdot \frac{\partial \mathbf{p}(t)}{\partial t}\\
    &= \frac{\partial \mathcal{H}}{\partial \mathbf{q}(t)}\cdot \frac{\partial \mathcal{H}}{\partial \mathbf{p}(t)} - \frac{\partial \mathcal{H}}{\partial \mathbf{p}(t)}\cdot\frac{\partial \mathcal{H}}{\partial \mathbf{q}(t)} = 0.
\end{align}
This concludes that $\mathcal{H}(\mathbf{q}(t),\mathbf{p}(t))$ is conserved over time.
\end{proof}
For node message passing, $\mathbf{q}(t_1)$ at the end of each layer is aggregated through a normalized adjacency matrix $\tilde{A}$ by
\begin{align}
    \mathbf{q}(t_1) \gets (\tilde{A} + I_n)\mathbf{q}(t_1). \label{hamgnn-agg}
\end{align}


\section{Hamiltonian GNN with Parameterized Symplectic Structure}

\begin{figure}[t]
\centering
\includegraphics[width=\linewidth]{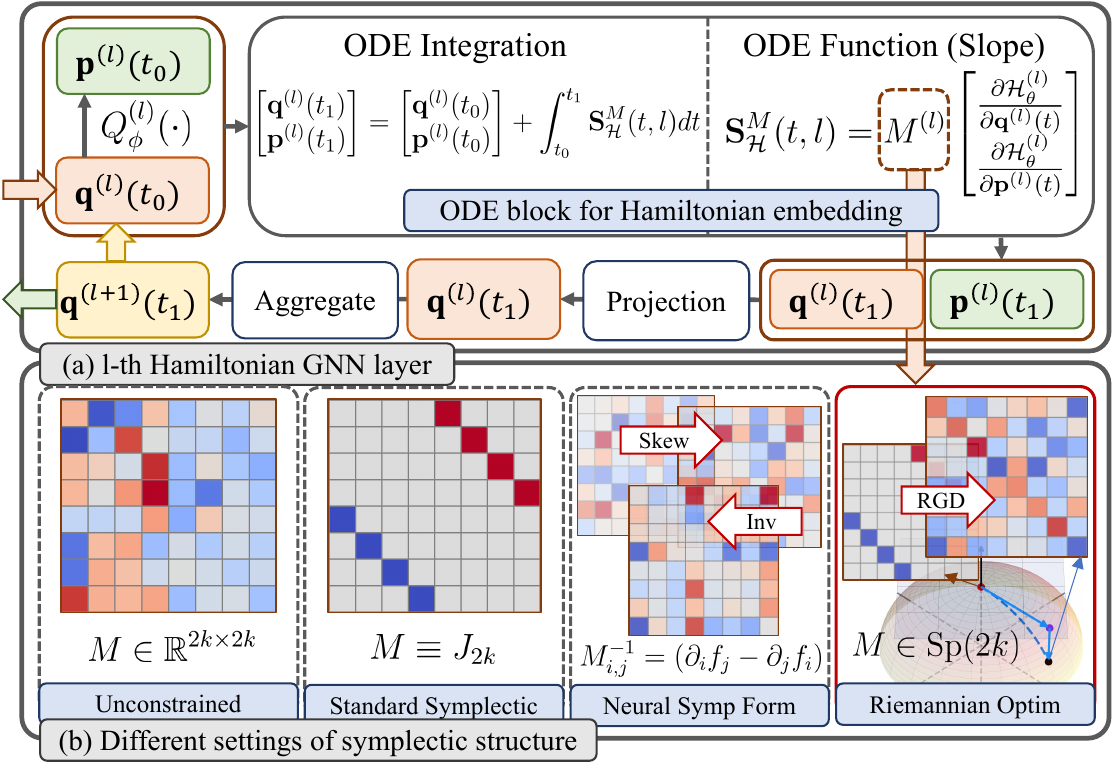}
\vspace{-0.3cm}
\caption{ \textit{(a)} Schematic of SAH-GNN. Essentially, it follows the process Eq.~(\ref{eq:hnn-integration}-\ref{hamgnn-agg}), we do not change the MLP nature of $\mathcal{Q}^{(l)}_\phi$ and $\mathcal{H}^{(l)}_\theta$. \textit{(b)} Different optimization method upon symplectic structural matrix $M$, where Riemannian gradient descent is leveraged by SAH-GNN during back-propagation.}
\label{Fig.pipeline-ham}
\end{figure}

\noindent The primary goal of a Hamiltonian Neural Network is to capture the underlying Hamiltonian dynamics of the system. One way to ensure this is to preserve the symplectic structure during the time evolution of the system. Ideally, the $\mathrm{Sp}$ manifold gives a space of matrices that preserve the symplectic structure. Motivated by this, we parameterize the symplectic matrix in Eq.~(\ref{eq:symplectic-gradient}) by $M^{(l)}\sim \mathrm{Sp}(2k^{(l+1)}, 2k^{(l)})$, such that the dynamic in Eq.~(\ref{eq:symplectic-gradient}) of $l$-th layer will now be governed by
\begin{align}
    &\mathbf{S}_{\mathcal{H}}^{M}(t,l) = M_{}^{(l)} 
    \begin{bmatrix}
        \frac{\partial \mathcal{H}}{\partial \mathbf{q}}\\
        \frac{\partial \mathcal{H}}{\partial \mathbf{p}}
    \end{bmatrix} , &\textit{s.t. } M_{}\in \mathrm{Sp}(2k^{(l+1)}, 2k^{(l)}). \label{eq:parameterized-dynamic}
\end{align}
\noindent\textbf{Advantage}: Allowing $M$ to be parameterized that deviate from the standard $J$ provides additional flexibility. The system can potentially capture non-standard dynamics or compensate for approximations in the learned Hamiltonian. Mathematically, $M\in\mathrm{Sp}$ guarantees the following properties

\begin{fact}
    For any integer $n>0$, $J_{2n}$ is a member of $\mathrm{Sp}(2n)$.
\end{fact}

\begin{proposition}[Preservation of Volume and Orientation]
\label{prop:volume-preserve}
    Any matrix $M\in \mathrm{Sp}$ has $+1$ determinant, indicating transformations by $M$ preserve volumes in the phase space. The orientation also preserves as it tied to $\mathrm{sign}(\det(M))$.
\end{proposition}
\begin{proof}
For simplicity, we assume $k^{(l)} \equiv k^{(l+1)}, \forall l$. Let \( M\in \mathrm{Sp}(2k) \), so that $M^\top J_{2k} M = J_{2k}$, where \(
J_{2k} = \begin{bmatrix}
0 & I_k \\
-I_k & 0
\end{bmatrix}
\) is the standard symplectic matrix.

As $J_{2k}$ is a skew-symmetric matrix, we sketch this proof leveraging the property of \textit{Pfaffian}. Explicitly, for a skew-symmetric matrix $A$ (\textit{i.e.} $A^\top = -A$), we have
\begin{align}
    \mathrm{pf}(A)^2 = \mathrm{det}(A).
\end{align}
With a useful property that
\begin{align}
    \mathrm{pf}(BAB^\top) = \mathrm{det}(B)\mathrm{pf}(A),
\end{align}
since $J_{2k}$ is skew-symmetric, we have the following
\begin{align}
    \mathrm{pf}(J_{2k}) &= \mathrm{pf}(M^\top J_{2k} M)\\
    &= \mathrm{det}(M) \mathrm{pf}(J_{2k}) \\
    &\Rightarrow \mathrm{det}(M) = +1.
\end{align}
This concludes that any matrix $M\in\mathrm{Sp}(2k)$ always has a determinant of $+1$; therefore, such transformations in phase space are both volume and orientation preserving.
\end{proof}

\begin{proposition}[Conservation of Symplectic Form]
The symplectic 2-form in the phase space coordinates $(\frac{\partial \mathcal{H}}{\partial \mathbf{q}}, \frac{\partial \mathcal{H}}{\partial \mathbf{p}})$ with transformation $J$ is expressed as $\omega = \sum_i d (\frac{\partial \mathcal{H}}{\partial \mathbf{q}})_i \wedge d(\frac{\partial \mathcal{H}}{\partial \mathbf{p}})_i$. Applying transformation $M\in \mathrm{Sp}$ conserves the $\omega$.
\label{prop:conserve-symp-form}
\end{proposition}
\begin{proof}
Let $\mathbf{u} = \frac{\partial\mathcal{H}}{\partial\mathbf{q}}$ and $\mathbf{v} = \frac{\partial\mathcal{H}}{\partial\mathbf{p}}$. The symplectic 2-form \( \omega \) in the original coordinates can be expressed by
\begin{align}
    &\omega = \sum_{i=1}^k d\mathbf{u}_i \wedge d\mathbf{v}_i,\label{2-form-origin}
\end{align}
where $\wedge$ denotes the wedge product. Let $\mathbf{z} = \begin{bmatrix}
    \mathbf{u}\\
    \mathbf{v}
\end{bmatrix}\in \mathbb{R}^{2k}$. Then $d\mathbf{z} = \begin{bmatrix}
    d\mathbf{u}\\
    d\mathbf{v}
\end{bmatrix}$ is a $2k$-dim column vector of the differentials. Then Eq.~(\ref{2-form-origin}) is identical to
\begin{align}
    \omega = d\mathbf{z}^\top J_{2k} d\mathbf{z}.
\end{align}
To see why this holds, consider the product
\begin{align}
    d\mathbf{z}^\top J_{2k} d\mathbf{z} = \sum_{i=1}^{2k} \sum_{j=1}^{2k} (d\mathbf{z}_i) J_{ij} (d\mathbf{z}_j).
\end{align}
Expanding this out, by using the form of $J_{2k}$, we get
\begin{align}
    d\mathbf{z}^\top J_{2k} d\mathbf{z} &= \sum_{i=1}^{k} (d\mathbf{v}_id\mathbf{u}_i - d\mathbf{u}_id\mathbf{v}_i)\\
    &= \sum_{i=1}^k d\mathbf{u}_i \wedge d\mathbf{v}_i = \omega.
\end{align}
Now, applying the symplectic transformation on the original coordinates gives
\begin{align}
    \Tilde{\mathbf{z}} = \begin{bmatrix}
        \Tilde{\mathbf{u}}\\
        \Tilde{\mathbf{v}}
    \end{bmatrix} = M \begin{bmatrix}
        \mathbf{u}\\
        \mathbf{v}
    \end{bmatrix} = M\mathbf{z}.
\end{align}
Therefore, $d\Tilde{\mathbf{z}} = M d\mathbf{z}$, and then the new symplectic 2-form $\Tilde{\omega}$ can be derived by
\begin{align}
    \Tilde{\omega} &= (d\Tilde{\mathbf{z}})^\top J_{2k} d\Tilde{\mathbf{z}}\\
    &= (M d\mathbf{z})^\top J_{2k} (M d\mathbf{z})\\
    &= d\mathbf{z}^\top M^\top J_{2k} M d\mathbf{z}.
\end{align}
Since $M$ is symplectic, $M^\top J_{2k} M = J_{2k}$. Therefore:
\begin{align}
    \Tilde{\omega} =  d\mathbf{z}^\top M^\top J_{2k} M d\mathbf{z} =  d\mathbf{z}^\top J_{2k} d\mathbf{z} = \omega.
\end{align}
This confirms that the symplectic 2-form $\omega$ is preserved under the transformation $M\in\mathrm{Sp}(2k)$.
\end{proof}


Preserving volume in phase space ensures the retention of information about the system's initial state, which enables traceability to initial conditions. Conservation of the symplectic form anchors the geometry of phase space, governing interactions between coordinates and their conjugate momenta. Such a structural invariance allows coordinate transformations without changing the inherent dynamics.

\noindent\textbf{Optimization}: This paper is aware of two methods to optimize the motion controlled by $M$: the hard-constrained \textit{Riemannian optimization} and soft-constrained \textit{loss regularization}. We focus on the former and provide tools upon Riemannian geometry in Sec.~\ref{sec:riemannian-optim}. The latter is achieved by imposing regularization $\| M_{}^\top J_{2k}M_{} - J_{2k}  \|_F^2$ on loss, which however does not guarantee the symplectic constraint during back-propagation. We also follow \cite{chen2021neural, kang2023node} and cover the \textit{neural symplectic form} approach in the experiment section.

\subsection{Random Symplectic Initialization}
A random symplectic matrix $M \in \mathrm{Sp(2n, 2k)}$ can be constructed via a randomly sampled matrix $\Lambda\in \mathbb{R}^{2k\times 2k}$ by
\begin{align}
\label{eq:random-construction}
    M = \begin{bmatrix}
        V_{1:k, :}\\
    0_{(n-k) \times 2k}\\
    V_{k+1:2k, :}\\
    0_{(n-k) \times 2k}
    \end{bmatrix}, \text{ where } V = \exp \left( 
    \begin{bmatrix}
        {\left(\Lambda^\top \Lambda\right)}_{k+1:2k,:} \\
        -{\left(\Lambda^\top \Lambda\right)}_{1:k,:}
    \end{bmatrix} \right).
\end{align}
\begin{proposition}
    $M$ generated by Eq.~(\ref{eq:random-construction}) satisfies the point set constraint of $\mathrm{Sp}(2n, 2k)$, \textit{i.e.} $M^\top J_{2n} M = J_{2k}$.
    \label{prop:random-symp-init}
\end{proposition}
\begin{proof}
Given $M = \begin{bmatrix}
        V_{1:k, :}\\
    0_{(n-k) \times 2k}\\
    V_{k+1:2k, :}\\
    0_{(n-k) \times 2k}
    \end{bmatrix}$ a $2n \times 2k$ matrix, let $M_1 = \begin{bmatrix}
        V_{1:k, :}\\
    0_{(n-k) \times 2k}
    \end{bmatrix}\in \mathbb{R}^{n\times 2k}$ and $M_2 = \begin{bmatrix}
        V_{k+1:2k, :}\\
    0_{(n-k) \times 2k}
    \end{bmatrix}\in \mathbb{R}^{n\times 2k}$, so that $M = \begin{bmatrix}
        M_1\\ M_2
    \end{bmatrix}$. Then $M^\top J_{2n} M$ can be written as
\begin{align}
    & M^\top J_{2n} M\\
    &= \begin{bmatrix}
        M_1^\top M_2^\top 
    \end{bmatrix}
    \begin{bmatrix}
        0 & I_n\\
        -I_n & 0
    \end{bmatrix}
    \begin{bmatrix}
        M_1\\ M_2
    \end{bmatrix}\\
    &= M_1^\top M_2 - M_2^\top M_1.\label{eq:MTJM-proof-1}
\end{align}
Also given $V = \exp \left( 
    \begin{bmatrix}
        {\left(\Lambda^\top \Lambda\right)}_{k+1:2k,:} \\
        -{\left(\Lambda^\top \Lambda\right)}_{1:k,:}
    \end{bmatrix} \right) \mathbb{R}^{2k\times 2k}$ where $\Lambda^\top \Lambda$ is an arbitrary symmetric matrix, $V$ is identically a squeezed $M$ without $0$ entries. Following Eq.~(\ref{eq:MTJM-proof-1})
\begin{align}
    M_1^\top M_2
    &= \begin{bmatrix}
        V_{1:k, :}^\top & 0_{(n-k)\times 2k}^\top
    \end{bmatrix} 
    \begin{bmatrix}
        V_{k+1:2k, :}\\
    0_{(n-k) \times 2k}
    \end{bmatrix}\\
    &= V_{1:k, :}^\top V_{k+1:2k, :},
\end{align}
gives
\begin{align}
    &M^\top J_{2n} M\\
    &=M_1^\top M_2 - M_2^\top M_1\\
    &= V_{1:k, :}^\top V_{k+1:2k, :} - V_{k+1:2k, :}^\top V_{1:k, :}\\
     & = \begin{bmatrix}
        V_{1:k, :}^\top V_{k+1:2k, :}^\top 
    \end{bmatrix}
    \begin{bmatrix}
        0 & I_k\\
        -I_k & 0
    \end{bmatrix}
    \begin{bmatrix}
        V_{1:k, :}\\ V_{k+1:2k, :}
    \end{bmatrix}\\
    & = V^\top J_{2k} V. \label{eq:proof-MJM-VJV}
\end{align}
Let $L = \Lambda^\top \Lambda$ and $Z = \begin{bmatrix}
        L_{k+1:2k,:}\\
        -L_{1:k,:}
    \end{bmatrix}$. Then $V = \mathrm{exp}(Z)$. One can easily verify that, with any symmetric matrix $L\in \mathrm{Sym}(2k)$, with a constructed matrix $Z = \begin{bmatrix}
        L_{k+1:2k,:}\\
        -L_{1:k,:}
    \end{bmatrix}$, we have
\begin{align}
    Z^\top J_{2k} + J_{2k} Z = 0, \label{eq:skew-symplectic-constraint}
\end{align}
which is a property of \textit{skew-symplectic} matrix $Z$.

Next, we prove the property that with any matrix $Z$ satisfying Eq.~(\ref{eq:skew-symplectic-constraint}), $\mathrm{exp}(Z)\in \mathrm{Sp}(2k)$. To show this, we take the matrix $X(t) = \mathrm{exp}(t Z)$ defined by the differential equation:
\begin{align}
    \frac{d X(t)}{dt} = \frac{d }{dt}\left(\mathrm{exp}(t Z)\right) = ZX(t).
\end{align}
We can verify the initial conditional ($t=0$) that $X(0) = I_{2k} \in \mathrm{Sp}(2k)$ since $I_{2k} J_{2k} I_{2k} = J_{2k}$. Now, to prove $\mathrm{exp}(Z)$ is symplectic, it suffices to show that $X(t)$ is symplectic for all $t$. This means we need to show
\begin{align}
    \forall t: X(t)^\top J_{2k} X(t) = J_{2k}.
\end{align}
As we already showed that $X(0)^\top J_{2k} X(0) = J_{2k}$, we further need to show
\begin{align}
    \frac{d}{dt}\left( X(t)^\top J_{2k} X(t) \right) = 0.\label{eq:zero-momentum-of-pde-proof}
\end{align}
To show Eq.~(\ref{eq:zero-momentum-of-pde-proof}), we expand the left-hand side using the product rule of differential equation
\begin{align}
    &\frac{d X(t)^\top}{dt} J_{2k} X(t) + X(t)^\top J_{2k} \frac{d X(t)}{dt}\\
    &= X(t)^\top Z^\top J_{2k} X(t) + X(t)^\top J_{2k} Z X(t)\\
    &= X(t)^\top (Z^\top J_{2k} + J_{2k} Z) X(t) \\
    & = X(t)^\top 0 X(t) = 0. 
\end{align}
This concludes the proof of Eq.~(\ref{eq:zero-momentum-of-pde-proof}), which concludes that $\forall t: X(t)\in\mathrm{Sp}(2k)$. Therefore, when $t=1$ as a special case, $X(1) = V = \mathrm{exp}(Z) \in \mathrm{Sp}(2k)$.

Following Eq.~(\ref{eq:proof-MJM-VJV}), since $V\in\mathrm{Sp}(2k)$, we have
\begin{align}
    M^\top J_{2n} M = V^\top J_{2k} V = J_{2k},
\end{align}
which concludes the proof.
\end{proof}

\subsection{Riemannian Symplectic Structure Optimization}
\label{sec:riemannian-optim}
Let us use randomly constructed $M^{(l)}\in \mathrm{Sp}(2k^{(l+1)}, 2k^{(l)})$ for each layer as starting points. Given an objective function $\mathcal{L}$, we regard each layer's $M^{(l)}$ in Eq.~(\ref{eq:parameterized-dynamic}) as a manifold parameter. Let $\Theta$ denote the Euclidean parameters. Riemannian optimizers can be defined to minimize $\mathcal{L}(\Theta, M)$ \textit{w.r.t.} $M$. In what follows, we illustrate the process of Riemannian gradient descent (RGD).

For each optimization step $s$, we first evaluate the Euclidean gradient \textit{w.r.t.} $M_s$ by $\nabla_{M_s}\mathcal{L}(\Theta, M_s)\in \mathbb{R}^{2k^{(l+1)}\times 2k^{(l)}}$, then for the gradient to satisfy the tangent space constraint in Eq.~(\ref{eq:tangent-set-sp}), we define the orthogonal projection of arbitrary Euclidean gradient $\mathbf{g}_s$ onto $\mathcal{T}_{M_s}\mathrm{Sp}(2k^{(l+1)}, 2k^{(l)})$.
\begin{definition}[Symplectic projection operator]
    Given an arbitrary matrix $P\in \mathrm{Sp}(2n, 2k)$. Define the projection matrix
\begin{align}
    \mathcal{P}_P := I_{2n} + P J_{2k} P^\top J_{2n}.
    \label{eq:projection_operator}
\end{align}
\end{definition}
\begin{proposition}
Given an arbitrary matrix $X\in \mathbb{R}^{2n\times 2k}$, a projection function $\Pi_{P}(X):\mathbb{R}^{2n\times 2k} \to \mathcal{T}_P\mathrm{Sp}(2n, 2k)$ to the tangent space can be realized through Eq.~(\ref{eq:projection_operator}) by
\begin{align}
    \Pi_{P}(X) := \mathcal{P}_P X = (I_{2n} + P J_{2k} P^\top J_{2n})X,
    \label{eq:projection_function}
\end{align}
subject to $\Pi_{P}(X)^\top J P + P^\top J \Pi_{P}(X) = 0$.
\label{prop:tangent-projector}
\end{proposition}
\begin{proof}
We verify that with arbitrary $X\in\mathbb{R}^{2n\times 2k}$, $\Pi_{P}(X) = (I_{2n} + P J_{2k} P^\top J_{2n})X$ is on the tangent space $\mathcal{T}_{P}\mathrm{Sp}(2n, 2k)$ defined by Eq.~(\ref{eq:tangent-set-sp}) on $P\in \mathrm{Sp}(2n, 2k)$. We start by expanding the tangent space constraint
\begin{align}
&\Pi_{P}(X)^\top J_{2n}X + X^\top J_{2n}\Pi_{P}(X)\\
&= X^\top(I_{2n} + P J_{2k} P^\top J_{2n})^\top J_{2n}X +\nonumber\\
&\qquad X^\top J_{2n} (I_{2n} + P J_{2k} P^\top J_{2n})X\\
&= X^\top ( \Delta ) X, \label{eq:XT-delta-X}
\end{align}
where
\begin{align}
    \Delta &= (P J_{2k} P^\top J_{2n})^\top J_{2n} + 2J_{2n} + J_{2n}PJ_{2k} P^\top J_{2n}\\
    &= J_{2n}^\top P J_{2k}^\top P^\top J_{2n} + 2J_{2n}+ J_{2n}PJ_{2k} P^\top J_{2n}\\
    &= \frac{J_{2n}^\top P J_{2k}^\top P^\top J_{2n}P + 2J_{2n}P+ J_{2n}PJ_{2k} P^\top J_{2n}P}{P}\label{eq:proof-verify-tangent-divide-p}.
\end{align}
Since $P\in \mathrm{Sp}(2n,2k)$, thus $P^\top J_{2n} P = 0$. Also given that $J^\top = -J$, $J^\top J=I$ and $JJ=-I$. Eq.~(\ref{eq:proof-verify-tangent-divide-p}) can be further derived as
\begin{align}
    &\frac{J_{2n}^\top P J_{2k}^\top (P^\top J_{2n}P) + 2J_{2n}P+ J_{2n}PJ_{2k} (P^\top J_{2n}P)}{P}\\
    &= \frac{J_{2n}^\top P J_{2k}^\top J_{2k} + 2J_{2n}P+ J_{2n}PJ_{2k} J_{2k}}{P}\\
    &= \frac{J_{2n}^\top P + 2J_{2n}P - J_{2n}P}{P}\\
    & = -J_{2n}  + 2J_{2n} - J_{2n} = 0.
\end{align}
Therefore, with $\Delta=0$, we can conclude that Eq.~(\ref{eq:XT-delta-X}) equals $0$. Hence, the constraint of the tangent space is satisfied.
\end{proof}

Then with the projected gradient $\Pi_{M_s}(\nabla_{M_s}\mathcal{L}(\Theta, M_s))$, the gradient update step is performed via local diffeomorphism between tangent space and the manifold. For $\mathrm{Sp}$ manifold, a retraction map $\mathcal{R}_P:\mathcal{T}_P\mathrm{Sp}(2n, 2k)\to$ $ \mathrm{Sp}(2n, 2k)$ is more commonly used instead of exponential map for computational efficiency. Given properly defined invertible $\mathcal{R}$, the Riemannian gradient update can be achieved via
\begin{align}
    M_{s+1} = \mathcal{R}_{M_{s}}(-\eta_s \Pi_{M_s}(\nabla_{M_s}\mathcal{L}(\Theta, M_s))),\label{eq:r-gradient-update}
\end{align}
where $\eta_s$ is the learning rate at gradient step $s$. To formulate $\mathcal{R}$, we first define the \textit{symplectic inverse} that is used to define such mapping.
\begin{definition}[Symplectic inverse]
    For any matrix $X\in\mathbb{R}^{2n\times 2k}$, the symplectic inverse is defined as
    \begin{align}
        X^+ := J_{2k}^\top X^\top J_{2n}.
        \label{eq:symplectic-inverse}
    \end{align}
\end{definition}
Given by \cite{bendokat2021real}, let $P\in \mathrm{Sp}(2n, 2d)$ and $X\in \mathcal{T}_P \mathrm{Sp}(2n,2k)$. With $B= P^+X$ and $H=X - PB$, the Cayley retraction on the real symplectic Stiefel manifold is expressed as
\begin{align}
    \mathcal{R}_P(\lambda X) := (\lambda H + 2P)\left( \frac{\lambda^2}{4}H^+ H - \frac{\lambda}{2}B + I_{2k} \right)^{-1} - P\label{eq:retr}.
\end{align}
In practice, we use Eq.~(\ref{eq:retr}) to perform Riemannian gradient update in Eq.~(\ref{eq:r-gradient-update}). We also give an analytical inverse to Eq.~(\ref{eq:retr}) to show its invertibility. Let $Q$ be another matrix on $\mathrm{Sp}(2n, 2k)$, the inversed retraction can be defined as
\begin{align}
    \mathcal{R}_P^{-1}(Q):= PB + H,\label{eq:invretr}
\end{align}
where
\begin{align}
    & B = 2 ( (I_{2k} + Q^+P)^{-1} - (I_{2k} + P^+Q)^{-1} ),\\
    & H = 2 ((P+Q) (I_{2k} + P^+Q)^{-1} - P ).
\end{align}
\begin{proposition}
    If $(I_{2k} + Q^+P)^{-1}$ and $(I_{2k} + P^+Q)^{-1}$ exist, it holds for $\mathcal{R}_P(\mathcal{R}_P^{-1}(Q)) = Q$ for any $P, Q\in \mathrm{Sp}(2n, 2k)$.
\label{prop:verify-inverse}
\end{proposition}
Given $P, Q\in\mathrm{Sp}(2n, 2k)$ and $X\in \mathcal{T}_{P}\mathrm{Sp}(2n, 2k)$, the retraction ($\lambda=1$) is expressed as
\begin{align}
    \mathcal{R}_P(X) := (H + 2P)\left( \frac{1}{4}H^+ H - \frac{1}{2}B + I_{2k} \right)^{-1} - P,
\end{align}
where
\begin{align}
    & B= P^+X,\\
    & H=X - PB.
\end{align}
If $(I_{2k} + Q^+P)^{-1}$ and $(I_{2k} + P^+Q)^{-1}$ exist, we have the inversed retraction
\begin{align}
    \mathcal{R}_P^{-1}(Q):= PB + H,
\end{align}
where $B$ and $H$ can be further expanded as
\begin{align}
    & B = 2 ( (I_{2k} + Q^+P)^{-1} - (I_{2k} + P^+Q)^{-1} ),\\
    & H = 2 ((P+Q) (I_{2k} + P^+Q)^{-1} - P ).
\end{align}
Also given some useful properties of symplectic inverse: suppose $A\in\mathbb{R}^{2n\times 2k}$, 
the following are equivalent
\begin{align}
    &A\in\mathrm{Sp}(2n,2k),\\ 
    &(A^+)^\top \in\mathrm{Sp}(2n,2k),\\
    &A^+A = I_{2k}.
\end{align}
Under these preliminaries, we prove $\mathcal{R}_P(\mathcal{R}_P^{-1}(Q)) = Q$.
\begin{proof}
First, we have
\begin{align}
    H = 2 ((P+Q) (I_{2k} + P^+Q)^{-1} - P ).
\end{align}
Taking the transpose of $H$ gives
\begin{align}
    H^\top = 2 (\left((I_{2k} + P^+Q)^\top \right)^{-1} (P^\top+Q^\top)  - P^\top ),
\end{align}
and thus the symplectic inverse of $H$ can be derived as
\begin{align}
    &H^+ = J_{2k}^\top H^\top J_{2n}\label{eq:derive-H+-0}\\
    &= 2 (J_{2k}^\top \left((I_{2k} + P^+Q)^\top \right)^{-1} (P^\top+Q^\top) J_{2n} - P^+ )\\
    &= 2(\underbrace{J_{2k}^\top \left((I_{2k} + P^+Q)^\top \right)^{-1} J_{2k}}_{\mathcal{F}} \underbrace{J_{2k}^\top (P^\top+Q^\top) J_{2n}}_{\mathcal{G}} -  P^+ ). \label{eq:derive-H+-1}
\end{align}
In Eq.~(\ref{eq:derive-H+-1}), we have
\begin{align}
    &\mathcal{F} = J_{2k}^\top \left((I_{2k} + P^+Q)^\top \right)^{-1} J_{2k} \\
    &= J_{2k}^\top \left( I_{2k} + Q^\top J_{2n}^\top PJ_{2k} \right)^{-1} J_{2k}\\
    &= (J_{2k})^{-1} \left( I_{2k} + Q^\top J_{2n}^\top PJ_{2k} \right)^{-1} (J_{2k}^\top)^{-1}\\
    &= \left( J_{2k}^\top J_{2k} + J_{2k}^\top Q^\top J_{2n}^\top PJ_{2k}J_{2k} \right)^{-1}\\
    &= \left( I_{2k} + J_{2k}^\top Q^\top (-J_{2n}) P (-I_{2k}) \right)^{-1}\\
    &= \left( I_{2k} + Q^+ P \right)^{-1},
\end{align}
and
\begin{align}
    &\mathcal{G} = J_{2k}^\top (P^\top+Q^\top) J_{2n}\\
    &=  J_{2k}^\top P^\top J_{2n}+ J_{2k}^\top Q^\top J_{2n}\\
    &= P^+ + Q^+.
\end{align}
Thus Eq.~(\ref{eq:derive-H+-0}) can be further simplified as
\begin{align}
    &H^+ = 2(\mathcal{F}\mathcal{G} - P^+) \\
    &= 2( ( I_{2k} + Q^+ P )^{-1} (P^+ + Q^+)  - P^+).
\end{align}
Letting $U = (I_{2k} + Q^+P)^{-1}$ and $V = (I_{2k} + P^+Q)^{-1}$, we have
\begin{align}
    &\frac{1}{4}H^+ H \label{eq:H+H-proof-start}\\
    &= (U(P^+ + Q^+) - P^+) ((P+Q)V - P)\\
    &= U(P^+ + Q^+)(P+Q)V - U(P^+ + Q^+)P\nonumber\\
    &\qquad - P^+(P+Q)V + P^+ P\\
    &= U(P^+ + Q^+)(P+Q)V - UU^{-1} - V^{-1}V + I_{2k}\\
    &= U(P^+P + Q^+P + P^+Q + Q^+Q) - I_{2k}\\
    &= U((I_{2k} + Q^+P) + (P^+Q + I_{2k}))V - I_{2k}\\
    &= U(U^{-1} + V^{-1})V- I_{2k}\\
    &= U+ V - I_{2k}.
\end{align}
Also given that
\begin{align}
    B &= 2 ( (I_{2k} + Q^+P)^{-1} - (I_{2k} + P^+Q)^{-1} )\\
    &= 2(U-V),
\end{align}
we have
\begin{align}
    &\frac{1}{4}H^+H - \frac{1}{2}B + I_{2k} \\
    &= V + U - I_{2k} - U + V + I_{2k}\\
    &= 2V.
\end{align}
Therefore
\begin{align}
    &\mathcal{R}_P(\mathcal{R}^{-1}_P (Q))\\
    &= (H + 2P)(\frac{1}{4}H^+H - \frac{1}{2}B + I_{2k})^{-1} - P\\
    &= \frac{1}{2}(H+2P)V^{-1} - P\\
    &= \frac{1}{2}(H+2P)(I_{2k} + P^+Q) - P\\
    &= \frac{1}{2}H(P^+Q + I_{2k}) + P(P^+Q + I_{2k}) - P\\
    &= ((P+Q)(I_{2k} + P^+Q)^{-1} -P) (P^+Q + I_{2k}) \nonumber\\
    &\qquad +P(P^+Q + I_{2k}) - P\\
    &= P+ Q - P = Q.
\end{align}
This concludes the proof.
\end{proof}
\section{Experiments}
\subsection{Experimental Setup}
In this section, we investigate SAH-GNN compared to different HamGNN variants: HamGNN(U) (unconstrained Euclidean parameterized $M$, Fig.~\ref{Fig.pipeline-ham}(b1)), HamGNN(S) (standard symplectic, Fig.~\ref{Fig.pipeline-ham}(b2)), HamGNN(N) (neural symplectic form $M$ as Eq.~(5) in \cite{chen2021neural}, Fig.~\ref{Fig.pipeline-ham}(b3)) and SAH-GNN (hard symplectic manifold constrained-Riemannian optimization on $M$, Fig.~\ref{Fig.pipeline-ham}(b4)). 

\noindent\textbf{Dataset:} We select datasets with various geometries, including citation networks: \textsc{Cora}, \textsc{PubMed}, \textsc{CiteSeer} \cite{namata2012query, giles1998citeseer, sen2008collective}; and low hyperbolicity graphs: \textsc{Disease}, \textsc{Airport} \cite{chami2019hyperbolic}. 

\noindent\textbf{Parameter:} We use Adam optimizer to train the Euclidean parameters and implement Riemannian gradient descent method to optimize matrix $M^{(l)}$ in each layer. We employ RK4 integrator for ODE integration. We use optuna framework to tune the hyper-parameters for all models.

\subsection{Model Performance}
\begin{table}[t]
\centering
\caption{Test accuracy (\%) for node classification. The best, second best, and third best results for each criterion are highlighted in {\color{red}\textbf{red}}, {\color{violet}\textbf{violet}}, and {\color{blue}\textbf{blue}}, respectively.}
\resizebox{\linewidth}{!}{\begin{tabular}{@{}llllllr@{}}
\toprule
Model             & $\textsc{Disease}$                               & $\textsc{Airport}$                               & $\textsc{PubMed}$                                & $\textsc{CiteSeer}$                              & $\textsc{Cora}$                                  & \multirow{2}{*}{\begin{tabular}[c]{@{}r@{}}Avg.\\ Rank\end{tabular}} \\
$\delta$          & $0$                                              & $1$                                              & $3.5$                                            & $5$                                              & $11$                                             &                                                                      \\ \midrule
MLP               & 50.0$_{\pm 0.0}$                                 & 76.9$_{\pm 1.8}$                                 & 58.1$_{\pm 1.9}$                                 & 58.1$_{\pm 1.9}$                                 & 57.1$_{\pm 1.2}$                                 & 16.4                                                                 \\
HNN               & 56.4$_{\pm 6.3}$                                 & 80.5$_{\pm 1.5}$                                 & 71.6$_{\pm 0.4}$                                 & 55.1$_{\pm 2.0}$                                 & 58.0$_{\pm 0.5}$                                 & 16.0                                                                 \\ \midrule
GCN               & 81.1$_{\pm 1.3}$                                 & 82.3$_{\pm 0.6}$                                 & 77.8$_{\pm 0.8}$                                 & 71.8$_{\pm 0.3}$                                 & 80.3$_{\pm 2.3}$                                 & 8.8                                                                  \\
GAT               & 87.0$_{\pm 2.8}$                                 & {\color[HTML]{0070C0} \textbf{93.0$_{\pm 0.8}$}} & 77.6$_{\pm 0.8}$                                 & 68.1$_{\pm 1.3}$                                 & 80.3$_{\pm 0.6}$                                 & 8.2                                                                  \\
SAGE              & 81.6$_{\pm 7.7}$                                 & 81.9$_{\pm 0.9}$                                 & 77.6$_{\pm 0.2}$                                 & 65.9$_{\pm 2.3}$                                 & 74.5$_{\pm 0.9}$                                 & 11.6                                                                 \\
SGC               & 82.8$_{\pm 0.9}$                                 & 81.4$_{\pm 2.2}$                                 & 76.8$_{\pm 1.1}$                                 & 70.9$_{\pm 1.3}$                                 & 82.0$_{\pm 1.7}$                                 & 9.2                                                                  \\ \midrule
HGNN              & 80.5$_{\pm 5.7}$                                 & 84.5$_{\pm 0.7}$                                 & 76.7$_{\pm 1.4}$                                 & 69.4$_{\pm 0.9}$                                 & 79.5$_{\pm 0.9}$                                 & 11.2                                                                 \\
HGCN              & {\color[HTML]{0070C0} \textbf{89.9$_{\pm 1.1}$}} & 85.4$_{\pm 0.6}$                                 & 76.4$_{\pm 0.9}$                                 & 65.8$_{\pm 2.0}$                                 & 78.7$_{\pm 0.9}$                                 & 10.4                                                                 \\
LGCN              & 88.4$_{\pm 1.8}$                                 & 88.2$_{\pm 0.2}$                                 & 77.4$_{\pm 1.4}$                                 & 68.1$_{\pm 2.0}$                                 & 80.6$_{\pm 0.9}$                                 & 8.6                                                                  \\
$\kappa$-GCN      & 82.1$_{\pm 1.1}$                                 & 87.9$_{\pm 1.3}$                                 & {\color[HTML]{0070C0} \textbf{79.2$_{\pm 0.7}$}} & {\color[HTML]{7030A0} \textbf{73.3$_{\pm 0.5}$}} & 81.1$_{\pm 1.5}$                                 & 6.0                                                                  \\
$\mathcal{Q}$-GCN & 71.8$_{\pm 1.2}$                                 & 89.7$_{\pm 0.5}$                                 & {\color[HTML]{FF0000} \textbf{81.3$_{\pm 1.5}$}} & {\color[HTML]{FF0000} \textbf{74.1$_{\pm 1.4}$}} & {\color[HTML]{FF0000} \textbf{83.7$_{\pm 0.4}$}} & 4.9                                                                  \\ \midrule
GRAND             & 74.5$_{\pm 1.2}$                                 & 60.0$_{\pm 1.6}$                                 & {\color[HTML]{7030A0} \textbf{79.3$_{\pm 0.5}$}} & 71.8$_{\pm 0.8}$                                 & {\color[HTML]{7030A0} \textbf{82.8$_{\pm 0.9}$}} & 7.7                                                                  \\
Vanilla ODE       & 71.8$_{\pm 18.8}$                                & 90.3$_{\pm 0.6}$                                 & 73.3$_{\pm 3.3}$                                 & 56.6$_{\pm 1.3}$                                 & 68.4$_{\pm 1.2}$                                 & 13.3                                                                 \\
HamGNN(U)         & 89.8$_{\pm 1.6}$                                 & 94.5$_{\pm 0.5}$                                 & 77.6$_{\pm 1.1}$                                 & 65.0$_{\pm 2.9}$                                 & 74.9$_{\pm 1.4}$                                 & 8.6                                                                  \\ \midrule
HamGNN(S)         & {\color[HTML]{7030A0} \textbf{91.2$_{\pm 1.4}$}} & {\color[HTML]{FF0000} \textbf{95.5$_{\pm 0.5}$}} & 78.1$_{\pm 0.5}$                                 & 70.1$_{\pm 0.9}$                                 & {\color[HTML]{0070C0} \textbf{82.1$_{\pm 0.8}$}} & 4.2                                                                  \\
HamGNN(N)         & 88.4$_{\pm 1.5}$                                 & {\color[HTML]{7030A0} \textbf{93.7$_{\pm 0.2}$}} & 78.6$_{\pm 0.3}$                                 & 71.5$_{\pm 1.4}$                                 & 81.2$_{\pm 0.6}$                                 & 5.1                                                                  \\ \midrule
SAH-GNN           & {\color[HTML]{FF0000} \textbf{94.5$_{\pm 1.5}$}} & {\color[HTML]{FF0000} \textbf{95.5$_{\pm 0.5}$}} & 78.4$_{\pm 0.6}$                                 & {\color[HTML]{0070C0} \textbf{72.1$_{\pm 0.7}$}} & {\color[HTML]{0070C0} \textbf{82.1$_{\pm 1.3}$}} & \textbf{2.8}                                                         \\ \bottomrule
\end{tabular}}
\label{tb:nc-result}
\vspace{-0.5cm}
\end{table}

We evaluate models on node classification tasks. A lower hyperbolicity $\delta$ indicates a more tree-like structure. For fairness, we adhere to the baselines that are compared in \cite{kang2023node}. These include \textit{MLP baselines}: MLP and HNN \cite{ganea2018hyperbolic}; \textit{Euclidean baselines}: GCN \cite{kipf2016semi}, GAT \cite{velivckovic2017graph}, GraphSAGE \cite{hamilton2017inductive}, and SGC \cite{wu2019simplifying}; \textit{Hyperbolic baselines}: HGNN, HGCN \cite{chami2019hyperbolic}, LGCN \cite{zhang2021lorentzian}; \textit{Product manifold baselines}: $\kappa$-GCN \cite{bachmann2020constant}, $\mathcal{Q}$-GCN \cite{xiong2022pseudo}; \textit{Continuous GNN baselines}: Vanilla ODE, GRAND \cite{chamberlain2021grand}. We report the average classification Acc. $\pm$ Std. 


As summarized in Tab.~\ref{tb:nc-result}, SAH-GNN demonstrates its versatility and robustness across a range of datasets with varying geometrical characteristics. Notably, for datasets with a lower hyperbolicity $\delta$, \textit{e.g.} \textsc{Cora}, which can be well embedded into Euclidean space, SAH-GNN achieves performance comparable to Euclidean GNNs, closely to GRAND. This is particularly evident in the \textsc{Cora} dataset, where SAH-GNN attains an accuracy of  $82.1\%$, closely aligning with the $82.8\%$ achieved by GRAND. In scenarios involving tree-like datasets, exemplified by \textsc{Disease}, SAH-GNN outperforms all hyperbolic GNNs and HamGNN variants, giving $94.5\%$ on \textsc{Airport}, which is substantially higher than the next best HamGNN(S) at $91.2\%$. This indicates its superior capability in handling complex, hierarchical data structures.

Furthermore, complex manifold approaches, \textit{e.g.} $\kappa$-GCN and $\mathcal{Q}$-GCN, while showing promising results in highly $\delta$ datasets, exhibit only average performance on those with lower $\delta$. For instance, $\mathcal{Q}$-GCN performs exceptionally well on the \textsc{PubMed} dataset ($81.3\%$) and the \textsc{CiteSeer} dataset ($74.1\%$), both exhibiting high hyperbolicity. However, its performance is moderate on \textsc{Cora}. On the other hand, SAH-GNN consistently adapts and performs well across all datasets, irrespective of their intrinsic geometrical nature.


\subsection{Energy Conservation and Training Stability}
\begin{figure}[t]
\centering
\subfigure[Cora]{
\label{Fig.sub.cora-energy}
\includegraphics[width=0.4\linewidth]{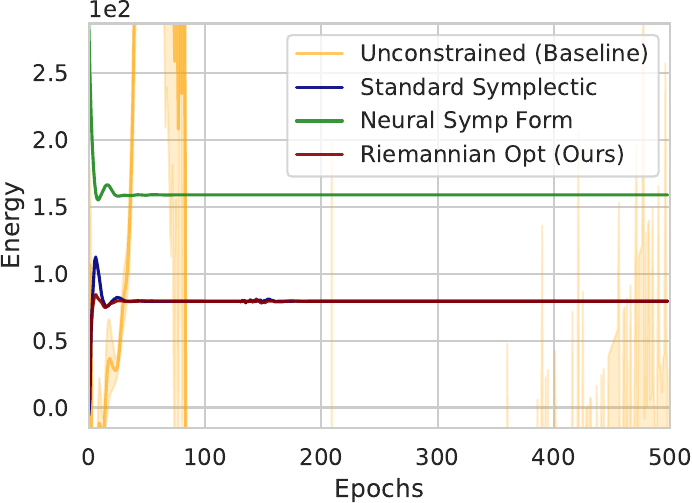}}
\subfigure[CiteSeer]{
\label{Fig.sub.citeseer-energy}
\includegraphics[width=0.4\linewidth]{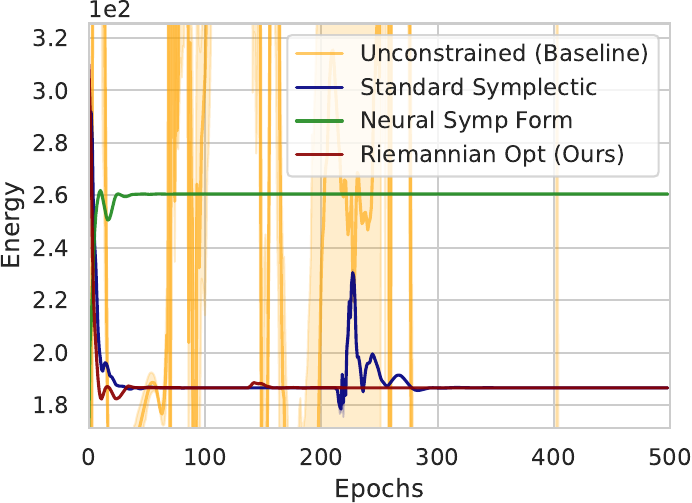}}

\vspace{-0.3cm}

\subfigure[Disease]{
\label{Fig.sub.disease-energy}
\includegraphics[width=0.4\linewidth]{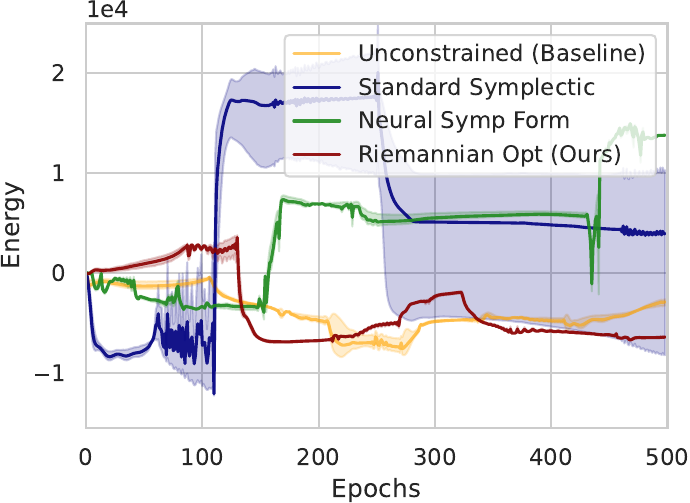}}
\subfigure[PubMed]{
\label{Fig.sub.pubmed-energy}
\includegraphics[width=0.4\linewidth]{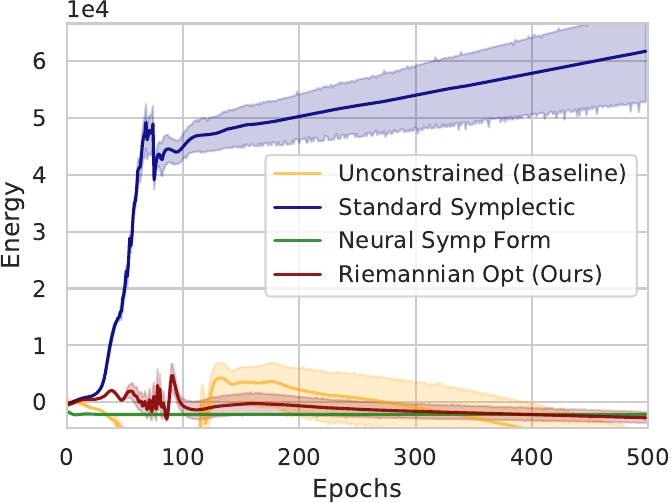}}


\caption{Variation of energy $\mathcal{H}_\theta$ through training via different optimization method of symplectic transformation $M$. A more stable curve is demanded as better conserved energy.}
\label{Fig.energy-convergence-ablation}
\vspace{-0.5cm}
\end{figure}

A learnable Hamiltonian $\mathcal{H}_\theta$ (typically a neural network $(\mathbf{q}, \mathbf{p})\to \mathbb{R}$) represents the energy of the system, parameterizing $\mathcal{H}_\theta$ enables better modeling of the physical system. According to Thm.~\ref{thm:conservation}, a physically meaningful system is reflected by a stable Hamiltonian. If $\mathcal{H}_\theta$ vibrates drastically during training, the dynamics $\mathbf{S}_\mathcal{H}$ induced by $\nabla_{(\mathbf{q}, \mathbf{p})} \mathcal{H}_\theta$ may also change dramatically, making it hard to train the network effectively and less reliable to predict future states.

Fig.~\ref{Fig.energy-convergence-ablation} visualizes the variation of energy during training with a total $500$ epochs and ODE step size $0.2$. Comparing Fig.~\ref{Fig.sub.cora-energy}-\ref{Fig.sub.citeseer-energy}, our method achieves the best stability, and slightly outperforms the standard symplectic in Fig.~\ref{Fig.sub.citeseer-energy} with a less oscillating curve. In Fig.~\ref{Fig.sub.disease-energy}, the standard symplectic exhibits large variance, and the curve of neural symplectic form oscillates within different intervals, while our method remains relatively stable. In Fig.~\ref{Fig.sub.pubmed-energy}, our method also converges towards a fixed value, whereas the energy of the standard symplectic explodes. All of this evidence suggests the prominence of Hamiltonian motion via parameterized $M$ in energy conservation.


We also show the loss curves for different $M$s in Fig.~\ref{Fig.loss-ablation}. Simply put, our approach exhibits lower loss values and less variance. This is indicative of better stability and expressivity of the underlying symplectic structure, assuming that parameterized symplectic transformations are physically meaningful.


\begin{figure}[h]
\centering
\subfigure[Training losses on Cora]{
\label{Fig.sub.cora-loss}
\includegraphics[width=0.4\linewidth]{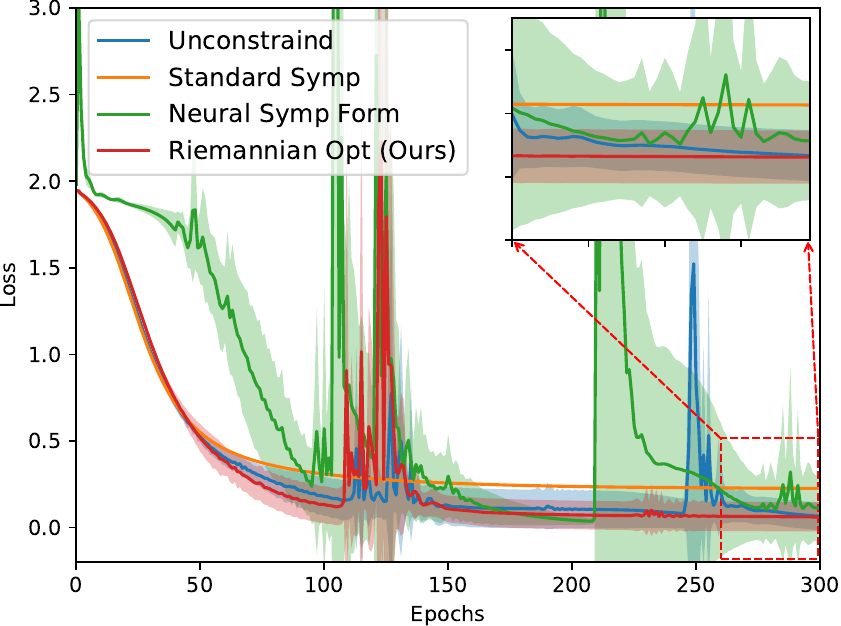}}
\subfigure[Training losses on CiteSeer]{
\label{Fig.sub.citeseer-loss}
\includegraphics[width=0.4\linewidth]{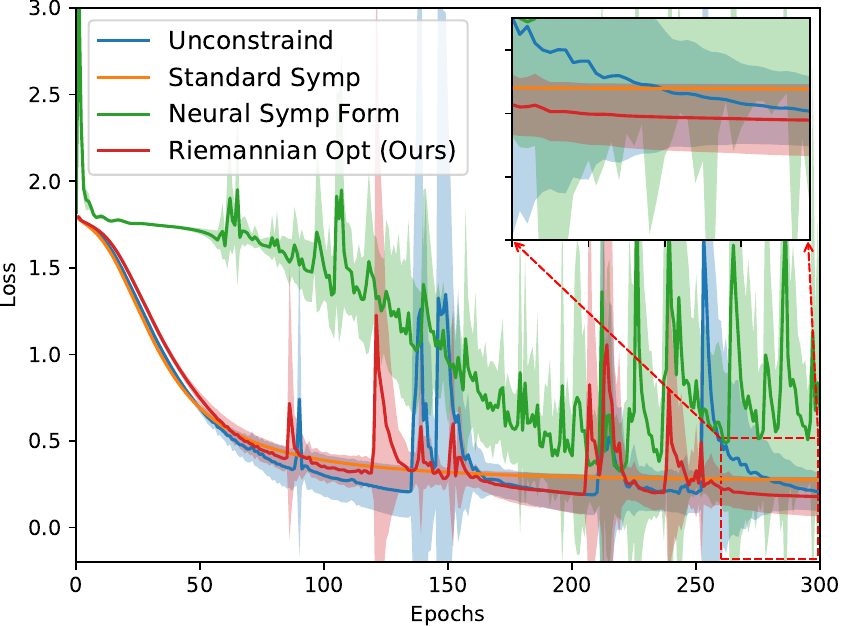}}
\vspace{-0.5cm}
\caption{Training losses on various datasets with 10 runs, we visualize the averaged curves with standard deviation.}
\label{Fig.loss-ablation}
\end{figure}

\vspace{-0.5cm}

\section{Conclusion}
This paper proposed SAH-GNN, a Hamiltonian Graph Neural Network that employs Riemannian optimization on the symplectic Stiefel manifold, for adaptive symplectic structure learning. Experimental results demonstrate that SAH-GNN not only outperforms existing methods in node classification tasks, but shows ability in energy conservation, which is particularly crucial for learning physically stable systems. Notably, the idea of learning symplectic structure with manifold constraint is versatile for Hamiltonian NN, which may open up an avenue for broader applications in various domains.

\vfill\clearpage
\bibliographystyle{IEEEbib}
\bibliography{refs}

\clearpage
\subsection*{APPENDIX}

\subsubsection*{EXPLANATION OF SYMBOLS}
We address the ambiguity of notation and symbols used in main content and proofs, detailed in Tab.~\ref{tb:summary-symbols}
\begin{table}[H]
\centering
\resizebox{0.99\linewidth}{!}{\begin{tabular}{@{}ll@{}}
\toprule
Symbol                                   & Explanation                                                                        \\ \midrule
$I_n$                                    & identity matrix of $n$ dimension                                                   \\
$J_{2n}$                                   & standard symplectic matrix, $J_{2n}=\begin{bmatrix}0 & I_n\\ -I_n& 0\end{bmatrix}$ \\
$\frac{\partial \mathbf{x}}{\partial t}$ & partial derivative of vector $\mathbf{x}$ with respect to $t$                      \\
$d\mathbf{x}$                            & total derivative of vector $\mathbf{x}$                                            \\
$\nabla$                                 & gradient                                                                           \\
$M_{a:b, c:d}$                           & slice of matrix $M$, from row $a$ to $b$ and column $c$ to $d$                     \\
$\mathbb{R}$                             & Euclidean space                                                                    \\
$\mathrm{Sp}$                            & symplectic group or symplectic Stiefel manifold                                    \\
$\mathcal{T}_{X}\mathrm{Sp}$             & tangent space of manifold $\mathrm{Sp}$ at position $X$                            \\
$\mathrm{exp}(X)$                        & exponential of each element in matrix $X$                                          \\
$\mathrm{det}(X)$                        & determinant of matrix $X$                                                          \\
$\mathrm{sign}(x)$                       & sign of scalar $x$, can be $+1$, $0$ or $-1$                                       \\
$\mathcal{L}$                            & loss function for optimization                                                     \\
$X^\top$                                 & transpose of a matrix $X$                                                          \\
$X^+$                                    & symplectic inverse of a matrix $X$, explained in Eq.~(\ref{eq:symplectic-inverse}) \\ \bottomrule
\end{tabular}}
\caption{Explanation of notation and symbols.}
\label{tb:summary-symbols}
\end{table}

\subsubsection*{DATASET STATISTICS}
We conduct experiments on five benchmark graph datasets, including citation networks (\textsc{Cora}, \textsc{PubMed}, \textsc{CiteSeer}) and tree-like graphs (\textsc{Disease}, \textsc{Airport}). They are widely used to test the performance of Riemannian geometric GNNs. We use the public (Planetoid) train/validation/test split for citation networks, and use the same split on tree-like graphs. Their data statistics are summarized in Tab.~\ref{tb:summary-dataset}.

\begin{table}[H]
\centering
\resizebox{0.9\linewidth}{!}{\begin{tabular}{lrrrrr} 
\toprule
Dataset  & \# Nodes & \# Edges & Classes & Features & $\delta$  \\ 
\midrule
\textsc{Disease}  & 1,044    & 1,043    & 2       & 1,000    & 0         \\
\textsc{Airport}  & 3,188    & 18,631   & 4       & 4        & 1         \\
\textsc{PubMed}   & 19,717   & 44,338   & 3       & 500      & 3.5       \\
\textsc{CiteSeer} & 3,327    & 4,732    & 6       & 3,703    & 5         \\
\textsc{Cora}     & 2,708    & 5,429    & 7       & 1,433    & 11        \\
\bottomrule
\end{tabular}}
\caption{Summarize of benchmark datasets statistics.}
\label{tb:summary-dataset}
\end{table}

\subsubsection*{COMPARED BASELINES}
To validate the efficacy of SAH-GNN in node classification tasks, we compare our model against the state-of-the-arts, including a number of prominent (Non-Euclidean/ODE-based)GNN variants, as baselines, summarized as follows

\subsubsection*{1. Euclidean Embedding Models}
\begin{itemize}
    \item \textbf{GCN/SGC/GAT.} Given a normalized edge weight matrix $\Tilde{P}$, let $w_{i,j}$ be the $ij$-th element of $\Tilde{P}$ \textit{i.e.,} the weight of edge $(i,j)\in \mathcal{E}$, then $\forall i:\sum_{j}w_{i,j} = 1$. GCN/SGC/GAT learns the node embedding by propagating messages over the $\Tilde{P}$. The difference is that GCN leverages the renormalization trick to obtain a fixed augmented Laplacian for all message propagation, SGC simplifies the GCN by decoupling hidden weights and activation functions from aggregation layers, and GAT learns dynamic $w_{i,j}$ which is normalized using softmax before each propagation.
    \item \textbf{SAGE.} GraphSAGE is an inductive framework that learns to generate node embeddings by sampling and aggregating features from a node's local neighborhood. Unlike methods that require the entire graph to generate embeddings, the SAGE approach enables generating embeddings for unseen nodes, making it particularly useful for dynamic graphs. It supports different aggregators, such as mean, LSTM, and pooling, to combine neighbor information and can be used for tasks like node classification and link prediction.
\end{itemize}

\subsubsection*{2. Non-Euclidean Embedding Models}
\begin{itemize}
    \item \textbf{HGNN/HGCN/LGCN.} To ensure the transformed features satisfy the hyperbolic geometry, HGNN/HGCN define feature transformation and message aggregation by pushing forward nodes to tangent space then pull them back to the hyperboloid. LGCN further enhances linear transformation and proposes Lorentz centroid message aggregation, which is an improved weighted mean operator that reduces distortion compared to local diffeomorphisms, such as tangent space operations.
    \item \textbf{$\kappa$-GCN.} Beyond hyperbolic space, $\kappa$GCN generalize embedding geometry to products of $\kappa$-Stereographic models, which includes Euclidean space, Poincar\'e ball, hyperboloid, sphere, and projective hypersphere.
    \item \textbf{$\mathcal{Q}$-GCN.} The $\mathcal{Q}$-GCN extends the hyperbolic embedding manifold to pseudo-Riemannian manifolds, accommodating graphs with mixed geometric structures. Embedding a graph to pseudo-Riemannian manifold for graph learning is particularly effective for handling complex graphs like social and molecular networks, capturing their heterogeneous topological structures and offering flexibility in modeling.
\end{itemize}

\subsubsection*{3. ODE-based Continuous Embedding Models}
\begin{itemize}
    \item \textbf{HamGNN.} HamGNN is a new GNN paradigm based on Hamiltonian mechanics. HamGNN models the embedding update of a node feature as a Hamiltonian orbit over time, which is propagated through ODE integration. This approach generalizes the hyperbolic exponential maps and allows the learning of the underlying manifold of the graph during training. The Hamiltonian orbits offer a more generalized embedding space compared to hyperbolic embedding. The architecture includes Hamiltonian layers where node features evolve along Hamiltonian orbits on a manifold, guided by a learnable Hamiltonian function (the parameterized neural network $\mathcal{H}_\theta$), and neighborhood aggregation is performed after each Hamiltonian layer. For more details, we refer the readers to \cite{greydanus2019hamiltonian,kang2023node}.
    \item \textbf{GRAND.} GRAND consider the \textit{heat} diffusion instead of graph convolution as the ODE function. The diffusion process describes a directional node embedding (heat) variation induced by neighboring positions; the strength of such variation is called diffusivity, which is modeled by multi-head attention. In such a way, the attentive aggregation is applied on \textit{difference of node pairs} rather than directly on node embeddings.
\end{itemize}

\subsubsection*{HYPERPARAMETERS}
We follow the same experiment setting as \cite{chami2019hyperbolic} and list the hyperparameters for the five datasets in Tab.~\ref{tb:hyperparams}.
\begin{table}[H]
\centering
\resizebox{0.99\linewidth}{!}{\begin{tabular}{@{}llllll@{}}
\toprule
Task          & \textsc{Disease} & \textsc{Airport} & \textsc{PubMed} & \textsc{CiteSeer} & \textsc{Cora} \\ \midrule
Learning Rate & 0.01             & 0.01             & 0.01            & 0.01              & 0.01          \\
Weight Decay  & 0.0001           & 0.0001           & 0.001           & 0.001             & 0.001         \\
Dropout       & 0                & 0                & 0.1             & 0.1               & 0             \\
Hidden Dim    & 128              & 128              & 32              & 64                & 64            \\
Layer        & 1                & 2                & 1               & 1                 & 1             \\
ODE Time & 1              & 1              & 1             & 1               & 1           \\
ODE Step Size & 0.2              & 0.5              & 0.2             & 0.2               & 0.2           \\
Max Grad Norm & 0.5              & 1                & 0.5             & 1                 & 1             \\ \bottomrule
\end{tabular}}
\caption{Hyperparameters for network embeddings.}
\label{tb:hyperparams}
\end{table}

\subsubsection*{COMPLEXITY ANALYSIS}
Our method have the \textit{same inference time} compared to HamGNN(S/U) and \textit{slightly more training time} which is comparable to HamGNN(N). Assume the structure matrix $M\in \mathrm{Sp}(2k)$ and $n=1$ for simplicity (where $k$ is the hidden dimensionality and $n$ is the number of data). During inference, the parameters matrix $M$ is fixed, hence is the same as using standard symplectic matrix $J$ and does not contribute to overall complexity. Recall Eq.~(\ref{eq:r-gradient-update}) during the traning process, the operations can be divided into three major procedures, respectively,
\begin{itemize}
    \item \textbf{Euclidean Gradient Computation.} The time complexity for computing the $\nabla_M \mathcal{L}$ depends on the specific operations involved in the loss function $\mathcal{L}$. Then this is typically $O(k^2)$ to $O(k^3)$ for matrix operations.
    \item \textbf{Projection onto the Tangent Space.} The projection operator in Eq.~(\ref{eq:projection_operator}) and projection function in Eq.~(\ref{eq:projection_function}) requires only matrix addition and multiplication, thus can be viewed as $O(k^3)$ complexity.
    \item \textbf{Retraction Map Computation.} In Eq.~(\ref{eq:retr}), the most computationally intensive part is the matrix inversion $\left( \frac{\lambda^2}{4}H^+ H - \frac{\lambda}{2}B + I_{2k} \right)^{-1}$, since all matrices are in $2k$-dimensional, thus the inversion is of $O(k^3)$ complexity.
\end{itemize}
Therefore, the overall time complexity of the Riemannian gradient descent is dominated by the matrix multiplication and inversion steps, so theoretically for CPU computation, the complexity is $O(k^3)$, which is the same as vanilla HamGNN. However, the matrix addition and multiplication are easily scalable on GPU with CUDA, where as matrix inversion is hard to parallelize on GPU. Hence, our Riemannian gradient approach have relatively slower training time compared to our baseline vanilla HamGNN.

For memory complexity, our model only introduce $M$ as optimizable parameter, which has the same memory requirement as the neural symplectic form-HamGNN, with additional storage of parameter matrix of size $2k\times 2k$. Since $k$ is small (typically $16\sim 128$ for graph learning), the extra memory usage can be disregarded.

\end{document}